\documentclass[conference]{IEEEtran}
\IEEEoverridecommandlockouts
\usepackage{cite}
\usepackage{amsmath,amssymb,amsfonts}
\usepackage{algorithmic}
\usepackage{graphicx}
\usepackage{textcomp}
\usepackage{xcolor}

\usepackage{algorithm}
\usepackage{algorithmic}
\usepackage{amsthm}
\usepackage{subfigure}
\usepackage{booktabs}
\newtheorem{theorem}{Theorem}

\newtheorem{corollary}[theorem]{Corollary}

\newtheorem{property}{Property}
\newtheorem{remark}{Remark}

\DeclareMathOperator*{\argmin}{arg\,min}

\newcommand{\A}{\mathcal{A}}

\def\BibTeX{{\rm B\kern-.05em{\sc i\kern-.025em b}\kern-.08em
    T\kern-.1667em\lower.7ex\hbox{E}\kern-.125emX}}
    
\begin{document}

\title{Safe Screening Rules for Group OWL Models}

\author{
    \IEEEauthorblockN{Runxue Bao$^{1*}$\thanks{*Corresponding Author: Runxue Bao, Email: runxue.bao@pitt.edu
},
    Quanchao Lu$^{2}$,
    Yanfu Zhang$^{3}$}
    
    \IEEEauthorblockA{$^{1}$Department of Electrical and Computer Engineering, University of Pittsburgh, Pittsburgh, PA, United States}
    
    \IEEEauthorblockA{$^{2}$School of Computer Science, Georgia Institute of Technology, Atlanta, GA, United States}
    \IEEEauthorblockA{$^{3}$Department of Computer Science, the College of William and Mary, Williamsburg, VA, United States}

    \IEEEauthorblockA{
        runxue.bao@pitt.edu,
        qlu43@gatech.edu,
        yzhang105@wm.edu
    }
}
\maketitle

\begin{abstract}
Group Ordered Weighted $L_{1}$-Norm (Group OWL) regularized models have emerged as a useful procedure for high-dimensional sparse multi-task learning with correlated features. Proximal gradient methods are used as standard approaches to solving Group OWL models. However, Group OWL models usually suffer huge computational costs and memory usage when the feature size is large in the high-dimensional scenario. To address this challenge, in this paper, we are the first to propose the safe screening rule for Group OWL  models by effectively tackling the structured non-separable penalty, which can quickly identify the inactive features that have zero coefficients across all the tasks. Thus, by removing the inactive features during the training process, we may achieve substantial computational gain and memory savings. More importantly, the proposed screening rule can be directly integrated with the existing solvers both in the batch and stochastic settings. Theoretically, we prove our screening rule is safe and also can be safely applied to the existing iterative optimization algorithms.  Our experimental results demonstrate that our screening rule can effectively identify the inactive features and leads to a significant computational speedup without any loss of accuracy. 
\end{abstract}


\section{Introduction}
Multi-task learning \cite{ben2003exploiting, micchelli2005kernels, maurer2016benefit, lin2019pareto} has received great attention in machine learning recently, which aims to learn the shared knowledge among multiple related tasks. Empirical and theoretical studies \cite{evgeniou2004regularized,zhou2011malsar,misra2016cross,sener2018multi} have shown regularized multi-task learning, which allows knowledge transfer between these tasks by encoding task relatedness with different regularization terms, can improve the generalization performance of learning tasks. Multi-task 
Lasso \cite{argyriou2007multi,lounici2009taking,obozinski2011support,lounici2011oracle} encourages all the tasks to select a common set of nonzero features by the $L_{2,1}$ norm. 
Trace norm is used to constrain the models from different tasks to share a low-dimensional subspace\cite{abernethy2009new}. However, correlated features widely exist in high-dimensional sparse learning. These models tend to arbitrarily select only one of the highly correlated features and thus can be unstable and difficult to interpret.

To effectively handle correlated features in multi-task learning, a new family of regularizers, termed Group Ordered Weight $L_{1}$-Norms \cite{oswal2016representational}, has emerged as a useful procedure for high-dimensional sparse regression recently, which can identify precise grouping structures of strongly correlated covariates automatically during the learning process without any prior information of feature groups.  Explicitly, the Euclidean norm of the features in the Group OWL penalty are sorted in a non-increasing order and the corresponding weights are also non-increasing.  The Group OWL regularizer encourages the coefficients associated with strongly correlated features to be equal, which is helpful to denoise features and improve the predictions\cite{oswal2016representational} . Owing to its effectiveness, Group OWL models are widely used in various kinds of applications, e.g., matrix recovery \cite{zeng2014group}, fMRI analysis \cite{oswal2016representational}, and neural networks training \cite{zhang2018learning}.

Although the Group OWL penalty is an important regularizer that can find the grouping structures automatically, the computational burden and memory usage with the structured non-separable penalty can be very heavy, especially in the high-dimensional setting. Thus, how to efficiently solve the Group OWL models is still an important topic to study. By exploiting the sparsity of the solution, the screening rule first introduced by \cite{Laurent2012safe} is a promising approach to eliminating inactive features whose coefficients are proved to be zero at the optimum and further speeding up the model training.  

Recently, many screening rules have been proposed to accelerate sparse learning models in single-task learning. The static safe screening rules in \cite{Laurent2012safe} for generalized $L_{1}$ problems eliminate features only once before any optimization algorithm. Relaxing the safe rule, strong rules \cite{tibshirani2012strong}, use heuristic strategies at the price of possible mistakes, which requires difficult post-processing to check for features possibly wrongly discarded. Another road to screening method is called sequential safe rules \cite{wang2013lasso,xiang2016screening}, which relies on the exact dual optimal solution and thus could be very time-consuming and lead to unsafe solutions in practice. More recently, many dynamic screening rules, which are conducted during the whole learning process based on the duality gap, are proposed in \cite{fercoq2015mind,shibagaki2016simultaneous,ndiaye2016gap,rakotomamonjy2019screening,bao2020fast,bao2022accelerated,ijcai2022p266} for a broad class of problems with both good empirical and theoretical results.

In the multi-task learning setting, dynamic screening rules are proposed for the generalized sparsity-enforcing models with the $L_{2,1}$ norm \cite{ndiaye2015gap,ndiaye2017gap}. The screening rule in \cite{wang2015safe} is developed for $L_{2,1}$ norm regularized multi-tasking learning with multiple data matrices. \cite{zhou2015safe} introduces a safe subspace screening rule for the multi-task problems with the nuclear norm. It shows that safe screening rules have been proposed to accelerate training algorithms for many sparse learning models by screening inactive features while all of them are limited to either single-task learning settings or the separable matrix norm in multi-task learning. So far there are still no safe screening rules proposed for Group OWL models. This is because of the over-complex Group OWL regularizer, which is non-separable and structured, and thus all the hyperparameters for the variables in each row of the Group OWL regularizer are unfixed until we finish the whole learning process. Due to the challenges above in tackling the structured non-separable Group OWL regularizer with numerous unfixed hyperparameters in multi-task learning, speeding up Group OWL models via screening rules is still an open and challenging problem.

To address these challenges, in this paper, we propose a safe dynamic screening rule for the family of Group OWL regularized models, which can be specialized to the popular multi-task regression and multinomial logistic regression. To the best of our knowledge, this work is the first attempt in this direction. Specifically, our screening rule is doubly dynamic with a decreasing left term by tracking the intermediate duality gap and an increasing right term by exploring the order of the primal solution with the unknown order structure via an iterative strategy. In high-dimensional multi-task learning, superfluous features with zero coefficients widely exist and thus our screening rule can effectively identify the inactive features in each iteration and then accelerate the original algorithms by eliminating these useless features. More importantly, the proposed screening rule can be easily incorporated into any existing standard iterative optimization algorithms since it is independent of the solver. Theoretically, we rigorously prove that our screening rule is safe when applied to any iterative optimization algorithms. The empirical performance on eight real-world benchmark datasets shows our algorithms can lead to significant computational gain, compared to standard proximal gradient descent methods.

\section{Group OWL Regularized Models}
Let $X = [x_{1},x_{2},\ldots,x_{n}]^\top \in \mathbb{R}^{n \times d}$ denotes the design matrix with $n$ observations and $d$ variables, $
Y \in \mathbb{R}^{n \times q}$ denotes the measurement matrix with $q$ tasks or classes,  $B \in \mathbb{R}^{d \times q}$ denotes the unknown coefficient matrix and $\Theta \in \mathbb{R}^{n \times q} $ denotes the solution of the dual problem, $\lambda = [\lambda_{1}, \lambda_{2}, \ldots, \lambda_{d}]$ is a non-negative regularization parameter vector of $d$ non-increasing weights that $\lambda_1 \geq \lambda_2 \geq \ldots \geq \lambda_d$, we consider Group OWL models by solving the minimization problem as follows:
\begin{eqnarray} \label{primal}
 {B^*} \in  \argmin\limits_{B \in \mathbb{R}^{d \times q} }  P_{\lambda}(B) := \underbrace{\sum_{i=1}^n f_i(x_i^\top B)}_{F(B)} + \underbrace{\sum_{i=1}^d\lambda_{i}\|B_{[i],:}\|_{2}}_{J_{\lambda}(B)},
\end{eqnarray}
where loss $F(B) = \sum_{i=1}^n f_i(x_i^\top B)$, $f_i: \mathbb{R}^{q}\rightarrow \mathbb{R}_+$ can be the squared loss for multi-task regression and the logistic loss for multi-class classification tasks with $\frac{1}{L}$-Lipschitz gradients, $J_{\lambda}(B) = \sum_{i=1}^d\lambda_{i}\|B_{[i],:}\|_{2}$ is the Group OWL penalty,   $\|B_{[1],:}\|_{2} \geq \|B_{[2],:}\|_{2}  \geq \cdots \geq \|B_{[d],:}\|_{2}  $ are the ordered terms. Each term has a corresponding regularization parameter. In general,  (\ref{primal}) is convex, non-smooth, and structured non-separable. 

Similar to the OWL regularizer \cite{lgorzata2013statistical, bao2019efficient}, the Group OWL regularizer penalizes the coefficients according to the magnitude of the $l_{2}$-norm: the larger the magnitude, the larger the penalty. Group OWL regression has been shown to outperform $L_{2,1}$-norm regularized models when the correlated features widely exist in the high-dimensional setting \cite{oswal2016representational}. It provides theoretical analyses that Group OWL regression can simultaneously encourage sparsity and grouping properties for sparse regression in multi-task feature learning with strongly correlated features. Note that the Group OWL models are a set of sparse multi-task learning models. For an illustration, we present the formation of Group OWL Regression and Multinomial OWL Regression.
\paragraph{Group OWL Regression}
For the multi-task regression task, loss $F(B)$ is defined as $F(B) :=\frac{1}{2}\|Y-X B\|^2_{2}$ and each $f_i$ is defined as $f_i(B) :=\frac{1}{2}\|Y_{i,:}-x_i^T B\|^2_{2}$, Group OWL regression is formulated as:
\begin{eqnarray} \label{groupowlregression}
    \min\limits_{B\in \mathbb{R}^{d \times q}}  P_{\lambda}(B):=\frac{1}{2}\|Y-XB\|^2_{2} +\sum_{i=1}^d\lambda_{i}\|B_{[i],:}\|_{2}.
\end{eqnarray}
2OSCAR \cite{zeng2014group} is a special case of Group OWL regression if $\lambda_{i} = \alpha_{1} + \alpha_{2}(d - i)$ for $i=1,2,\cdots,d$, where $\alpha_{1}$ and $\alpha_{2}$ are non-negative parameters. 

\paragraph{Multinomial OWL Regression}
For the multi-class classification, by defining loss $F(B)$ as multinomial logistic loss, multinomial OWL regression is formulated as follows:
\begin{eqnarray} \label{multilogisticregression}
\min\limits_{B\in \mathbb{R}^{d \times q}}  P_{\lambda}(B) \!\!\!\! &:= &\!\!\!\!\sum_{i=1}^n (\sum_{j=1}^q -Y_{i,j}x_i^\top B_{:,j} + \log(\sum_{j=1}^q\exp(x_i^\top B_{:,j}))) \nonumber\\
\!\!\!\!&&\!\!\!\! + \sum_{i=1}^d\lambda_{i}\|B_{[i],:}\|_{2}.
\end{eqnarray}

\begin{remark}
In practice, Group OWL models often involve high-dimensional data, making the aforementioned optimization algorithms computationally expensive and memory-intensive for large feature sizes $d$. Therefore, accelerating model training through screening techniques is both essential and promising.
\end{remark}

\section{Proposed Screening Rule}
In this section, we first derive the dual formulation and then propose  safe screening rules for  Group OWL models
by exploiting the properties of the screening test.

\subsection{Dual and Screening Test}
In this subsection, we derive the screening test for Group OWL models for $i = 1,2, \ldots, d$ as 
\begin{eqnarray} 
\|x_i^\top \Theta^* \|_2 < \lambda_{|\A^*|} \Longrightarrow B_{i,:}^* = 0,
\end{eqnarray}
where $B^{*}$ and $\Theta^*$ are the optimum of the primal and the dual respectively, $\A^*$ is the active set at the optimal with nonzero coefficients and $|\A^*|$ is the size of set $\A^*$.

To derive the screening test above, we provide the dual formulation for the primal objective (\ref{primal}) of Group OWL models. Following the derivation of $L_{2, 1}$-norm regularized regression in \cite{ndiaye2015gap,ndiaye2017gap}, we can derive the dual of  (\ref{primal}) as follows:
\begin{subequations}
\begin{small}
\begin{eqnarray} 
\!\!&  &\!\! \min\limits_{B} F(B)+J_{\lambda}(B) \label{yy} \\
\!\!& = &\!\!  \min\limits_{B}  \sum_{i=1}^{n} f_{i}(x_{i}^\top B) +\sum_{i=1}^d\lambda_{i}\|B_{[i],:}\|_{2} \nonumber \\
\!\!& = &\!\! \!\!\!\!\ \min\limits_{B}  \sum_{i=1}^{n} f^{\star\star}_{i}(x_{i}^\top B) +\sum_{i=1}^d\lambda_{i}\|B_{[i],:}\|_{2} \nonumber \\
\!\!& = &\!\! \min\limits_{B} \sum_{i=1}^{n} \max\limits_{\Theta_{i,:}} [\Theta_{i,:}(B^\top x_{i})-f^{\star}_{i}(\Theta_{i,:})] +\sum_{i=1}^d\lambda_{i}\|B_{[i],:}\|_{2} \nonumber \\
\!\!& = &\!\! \min\limits_{B} \max\limits_{\Theta} - \sum_{i=1}^{n} f^{\star}_{i}(\Theta_{i,:}) + \mathrm{tr}(\Theta B^\top X^\top) +\sum_{i=1}^d\lambda_{i}\|B_{[i],:}\|_{2} \nonumber \\
\!\!& = &\!\! \max\limits_{\Theta} - \sum_{i=1}^{n} f^{\star}_{i}(\Theta_{i,:}) + \min\limits_{B} \sum_{i=1}^{d} ( x_{i}^\top\Theta B_{i,:}^\top + \lambda_{i}\|B_{[i],:}\|_{2} )
         \label{maximization} \\
\!\!& = &\!\! \max\limits_{\Theta \in  \Delta}  D(\Theta) := 
 \sum_{i=1}^{n} - f^{\star}_{i}(\Theta_{i,:})  \label{dual} 
\end{eqnarray}
\end{small}
\end{subequations}
where $\mathrm{tr}$ is the trace of the matrix. Note that $f^{\star}_{i}$ is the convex conjugate of function $f_{i}$ as:
$
f^{\star}_{i}(\Theta_{i,:}) = \max\limits_{z_{i}} z_{i}\Theta_{i,:}^\top -f_{i}(z_{i}).$
Denote $\widetilde{\Theta}:= \big(\|x_1^\top \Theta\|_2, \|x_2^\top \Theta\|_2,\ldots, \|x_d^\top \Theta\|_2\big)^\mathsf{T}$, the constraint $\Theta \in  \Delta$ in (\ref{dual}) is $\sum_{j \leq i} \widetilde{\Theta}_{[j]} \leq \sum_{j \leq i} \lambda_j $ for all $i = 1, \ldots, d$. The last step to obtain the dual uses the optimality conditions of the minimization part in the penultimate formula of (\ref{maximization}) as: 
\begin{eqnarray} \label{optimality_origin}
\min\limits_{B} \sum_{i=1}^{d} x_{i}^\top\Theta B_{i,:}^\top + \lambda_{i}\|B_{[i],:}\|_{2},
\end{eqnarray}
which can be transformed as the constraints in (\ref{dual}) and completes the derivation of the dual formulation of (\ref{primal}).

Based on the Fermat rule \cite{bauschke2011convex}, we have 
\begin{eqnarray}  \label{subdiff}
    -X^\top \Theta^* \in \partial J_{\lambda}(B^*), 
\end{eqnarray}
where $\partial J_{\lambda}(B^*)$ is the subdifferential of $J_{\lambda}(B^*)$. 

Denote the active set and inactive set  at the optimal as $\A^*$ and $\widetilde{\A}^*$ respectively, we can derive the corresponding results for Problem (\ref{optimality_origin}) at partition $\A^*$ and  $\widetilde{\A}^*$  as:
\begin{eqnarray} \label{constraint2}
    -X_{\A^*,:}^\top\Theta^{*}   \in \partial J_{\lambda_{\A^*}}(B^*) ,   \\
      -X_{\widetilde{\A}^*,:}^\top\Theta^{*}   \in \partial J_{\lambda_{\widetilde{\A}^*}}(B^*).
\end{eqnarray}
For any $i \in \A^*$, we have $B_{i,:}^* \neq 0$ and 
\begin{eqnarray}  \label{optimal}
     \|x_i^\top \Theta^* \|_2  \in [\min_{j\in  \A^*}\lambda_j, \max_{j\in \A^*}\lambda_j]. 
\end{eqnarray}
Thus, suppose the optimum primal and dual solutions are known, note $\lambda = [\lambda_{1}, \lambda_{2}, \ldots, \lambda_{d}]$ is a non-negative vector that $\lambda_1 \geq \lambda_2 \geq \ldots \geq \lambda_d$, we can derive the screening condition for each variable as: 
\begin{eqnarray} \label{condition}
    \|x_i^\top \Theta^* \|_2 < \lambda_{|\A^*|} \Longrightarrow B_{i,:}^* = 0,
\label{condition00}
\end{eqnarray}
to identify the variables with zero coefficients. Then, in the latter training process, we can train the model with fewer parameters and variables without any loss of accuracy, which would speed up the training process. 

However, the main challenge is that the screening conditions (\ref{condition00}) require the dual optimum and the order structure of the primal optimum to be known in advance,  which will not happen until we finish the entire training process in practice. Thus, the optimization cannot benefit from the screening conditions above during the training process. 

Hence, the aim of our screening rule is to screen as many inactive variables whose coefficients should be zero as possible based on the screening test (\ref{condition00}) with the unknown dual optimum and the unknown order structure of the primal optimum during the optimization process. To achieve this, we can design safe screening rules by constructing a screening region as large as possible with smaller left lower bounds and larger upper bounds.

\subsection{Upper Bound for the Left Term}
Note the lower bound of the screening region is the upper bound of $\|x_i^\top \Theta^* \|_{2}$. Hence, in this part,  we derive a tight upper bound for $\|x_i^\top \Theta^* \|_{2}$ as
\begin{eqnarray} \label{upperbound}
 \|x_i^\top \Theta \|_{2}   +   \|x_i\|\sqrt{2G(B, \Theta)/L},
\end{eqnarray}
by tracking the intermediate duality gap $G(B, \Theta)$ at each iteration of the training process.

The derivations can be shown as follows. Considering triangle inequality, we have
\begin{eqnarray} \label{triangle}
\|x_i^\top \Theta^* \|_{2}\leq  \|x_i^\top \Theta \|_{2} +  \|x_i \|\|\Theta - \Theta^*\|_F. 
\end{eqnarray}

Note that each $f_i^*(\Theta_{i,:})$ in the dual is strongly convex (see Proposition 3.2 in\cite{johnson2015blitz}), we have Property \ref{property1}.
\begin{property} \label{property1}
Dual $D(\Theta):= \sum_{i=1}^{n} - f^{*}_{i}(\Theta_{i,:})$ is strongly concave \emph{w.r.t.} $\Theta$. Hence, we have
\begin{small}
\begin{eqnarray} \label{concave}
D(\Theta) \leq D(\Theta^*)- \mathrm{tr}(\nabla D(\Theta^*)^\top(\Theta^* - \Theta))   - \frac{L}{2}\| \Theta - \Theta^* \|^2_{F}.
\end{eqnarray}
\end{small}
\end{property}

Considering Property \ref{property1}, we can further bound the distance $\| \Theta - \Theta^* \|_F$ between the intermediate solution and the optimum of the dual based on the first-order optimality condition of constrained optimization in  Corollary \ref{corollary1}.

\begin{corollary} \label{corollary1}
Let $\Theta$ be any feasible dual, we have:
\begin{eqnarray}
    \| \Theta - \Theta^* \|_F \leq \sqrt{2G(B, \Theta)/L},
\end{eqnarray}
where $G(B, \Theta) = P(B) - D(\Theta)$ is the intermediate duality gap during the training process.
\end{corollary}

\begin{proof}
Considering the first-order optimality condition for strongly concave dual $D(\Theta)$, we have:
\begin{eqnarray} 
    \mathrm{tr}(\nabla D(\Theta^*)^\top(\Theta^* - \Theta)) \geq 0.
\end{eqnarray}
Hence, based on (\ref{concave}), we have:
\begin{eqnarray} 
\|\Theta-\Theta^*\|_F\leq\sqrt{\frac{2(D(\Theta^*) - D(\Theta))}{L}}.
\end{eqnarray}
By strong duality, we have $P(B ) \geq D(\Theta^*)$ and thus obtain the following bound as:
\begin{eqnarray} 
    \| \Theta - \Theta^* \|_F  \leq \sqrt{\frac{2(P(B ) - D(\Theta))}{L}},
\end{eqnarray}
which completes the proof. 
\end{proof}

Hence, with the obtained upper bound  in Corollary \ref{corollary1}, we can substitute $\|\Theta - \Theta^* \|_F$ in the right term of (\ref{triangle})  and then derive the new screening test with the new upper bound  as:
\begin{small}
\begin{eqnarray} \label{test}
\|x_i^\top \Theta \|_{2}   +   \|x_i\|\sqrt{2G(B, \Theta)/L} < \lambda_{|\A^*|}  \Rightarrow B_{i,:}^* = 0. 
\end{eqnarray}
\end{small}
The intermediate duality gap can be computed by $B$ and $\Theta$. $B$ and $\Theta$ can be easily obtained in the original proximal gradient algorithms. 

As the duality gap becomes smaller during the training process, the upper bound of $\|x_i^\top \Theta^* \|_{2}$ becomes smaller, and thus the lower bound of the screening region continuously becomes smaller.

 \subsection{Lower bound for the Right Term}
On the other hand, the upper bound of the screening region is the lower bound of $\lambda_{|\A^*|} $. Hence, in this part, we derive a tight lower bound of the right term as $\lambda_{|\A|}$ where $\A$ is defined as the active set that has not been screened out during the training process and $|\A|$ is the size of set $\A$.

To derive the tight lower bound of the right term above with the numerous unfixed hyperparameters,  we design an efficient iterative strategy to handle the non-separable penalty by exploring the unknown order structure of the primal optimum in (\ref{condition00}). In general, our proposed screening rule can be shown as 
\begin{eqnarray} \label{screening}
\|x_i^\top \Theta \|_{2}   +   \|x_i \|\sqrt{2G(B, \Theta)/L} < \lambda_{|\A|} \Rightarrow B^*_{i,:} = 0.
\end{eqnarray}

At first, the size of $\A$ is $d$ and the screening test is performed on $\lambda_d$ as 
$
\|x_i^\top \Theta \|_{2}   +   \|x_i\|\sqrt{2G(B, \Theta)/L} < \lambda_d \Rightarrow B^*_{i,:} = 0. 
$
During the training process with $d_k$ active features, the size of $\A$ is $d_k$. Hence, we can assign an arbitrary permutation of $d-d_k$ smallest parameters $\lambda_{d_k+1}, \lambda_{d_k+2}, \ldots, \lambda_{d}$ to these screened coefficients without any influence to the final model. Thus, the order of these variables whose coefficients must be zero is known to be $d-d_k$ minimal absolute values of all. Correspondingly, the screening test is performed on $\lambda_{d_k}$ as 
$
\|x_i^\top \Theta \|_{2}   +   \|x_i\|\sqrt{2G(B, \Theta)/L} < \lambda_{d_k} \Rightarrow B^*_{i,:} = 0, 
$
to find new set $\A$ with $d'_{k}$ active variables where $d'_{k} \leq d_{k}$ and further derive the order of $d_{k}-d'_{k}$ screened variables by assigning the parameters similarly as above. 

At each iteration, we repeat the screening test to explore the order of primal optimum until the active set remains unchanged. Note, for each time, we only need to do the screening test on a hyperparameter. That is to say, we only need to do the screening test once each time, independent of the size of the hyperparameters and thus very efficient to handle numerous hyperparameters of Group OWL models. 

We do the screening test above during the whole training process. As the size of $\A$  becomes smaller, the hyperparameter for the screening test becomes larger since $\lambda$ is a parameter vector with non-increasing weights. Thus, the lower bound of $\lambda_{|\A^*|}$ becomes larger and the upper bound of the screening region continuously becomes larger.

With our dynamic strategy on the both lower bound and upper bounds of the screening test, the screening region is continuously growing from both directions during the training process. Thus, our screening rule is promising to effectively screen more inactive variables.

\section{Proximal Gradient Algorithms with Safe Screening Rules}
In this section, we first apply the proposed safe screening rules to the proximal gradient algorithms, such as the APGD algorithm in the batch setting and SPGD algorithm in the stochastic setting, for Group OWL  models and then provide the theoretical analyses for our safe screening rules.

\subsection{Proposed Algorithms}

For the APGD algorithm in the batch setting, we repeat doing the screening test and updating active set $\A$ first. Then, we set the step size $t_k = t_1$ if $\A$ is updated at this iteration. For the following part, the procedure is the same as the original APGD algorithm with the current set $\A$. The procedure of our algorithm in the batch setting is summarized in Algorithm \ref{algAPGDScreen}. For the SPGD algorithm in the stochastic setting, similarly, we repeat doing the screening test and updating $\A$ in the outer loop first. After that, we derive the active set. Then, the procedure is the same as the original SPGD algorithm with the new active set. The procedure of our algorithm in the stochastic setting is summarized in Algorithm \ref{algSPGDScreen}.

\begin{algorithm}[ht] 
\renewcommand{\algorithmicrequire}{\textbf{Input:}}
\renewcommand{\algorithmicensure}{\textbf{Output:}}
\caption{APGD with Our Safe Screening Rules}
\begin{algorithmic}[1]
\REQUIRE $B^{0},\hat{B}^{1}=B^{0},t_{1}=1$.
\FOR{$k=1,2,\ldots$}
\REPEAT 
\STATE Do the screening test based on (\ref{screening}).
\STATE Update $\A$.
\UNTIL $\A$ keeps unchanged.
\IF{$\A$ changes}
\STATE $t_{k} = t_{1}$.
\ENDIF
\STATE $B^{k} = prox_{t_{k},{\lambda}}(\hat{B}^{k}-t_{k}\nabla F(\hat{B}^k))$.
\STATE $t_{k+1}=\frac{1}{2}(1+\sqrt{1+4 t_{k}^{2}})$.
\STATE $\hat{B}^{k+1} = B^{k} + \frac{t_{k}-1}{t_{k+1}}(B^{k}-B^{k-1})$.
\ENDFOR
\ENSURE Coefficient $B$.
\end{algorithmic}
\label{algAPGDScreen}
\end{algorithm}

\begin{algorithm}[ht]
\renewcommand{\algorithmicrequire}{\textbf{Input:}}
\renewcommand{\algorithmicensure}{\textbf{Output:}}
\caption{SPGD with Our Safe Screening Rules}
\begin{algorithmic}[1]
\REQUIRE $B^{0},l$.
\FOR{$k=1,2,\ldots$}
\REPEAT 
\STATE Do the screening test based on (\ref{screening}).
\STATE Update $\A$.
\UNTIL $\A$ keeps unchanged.
\STATE $B = B^{k-1}$.
\STATE $\tilde{v} = \nabla F(B)$.
\STATE $\tilde{B}^{0} = B$.
\FOR{$t=1,2,\ldots,T$}
\STATE Pick mini-batch $I_{t} \subseteq X$ of size $l$.
\STATE $v_{t}=(\nabla F_{I_{t}}(\tilde{B}^{t-1}) - \nabla F_{I_{t}}(B))/l + \tilde{v}  $.
\STATE $\tilde{B}^{t} = prox_{\gamma,{\lambda}}(\tilde{B}^{t-1}-\gamma v_{t})$.
\ENDFOR
\STATE $B^{k} = \tilde{B}^{T}$.
\ENDFOR
\ENSURE Coefficient $B$.
\end{algorithmic}
\label{algSPGDScreen}
\end{algorithm}

Interestingly, the duality gap, which is the main time-consuming step of our screening rule, has been computed by the stopping criterion evaluation of the original APGD and SPGD algorithms.  Further, we analyze the overall complexity of our algorithms here, which consists of the per-iteration cost and the number of iterations. First, in terms of the per-iteration cost,  our Algorithm \ref{algAPGDScreen} only requires the computation over the active set $\A$ with $d_k$ variables rather than the full set of the original APGD algorithm. Second, in terms of the number of iterations, since the optimal solution of the inactive features screened at iteration $k$ must be zeros, discarding inactive features in advance can only make the objective function remain the same or decrease. Thus, our Algorithm \ref{algAPGDScreen} will take at most the same (usually fewer in practice) iterations to convergence to the same stopping criterion, compared to the original APGD algorithm. With fewer or the same iterations and lower per-iteration costs, our proposed algorithm is more efficient than the original one. Similarly, our Algorithm  \ref{algSPGDScreen} also takes fewer or the same iterations and lower per-iteration cost, the overall complexity is less than the original SPGD algorithm. 

Specifically, the computational gain of our methods relies on the sparsity of the final model. The training benefits from the screening rule more with sparser models. When $n < d$, the final model would be very sparse and we have $d_{k} \ll d  $ during the training. Thus, our method is well suitable for high-dimensional settings.  It clearly shows the proposed algorithms are always much faster than the original ones. Meanwhile, note that our method also works for datasets with $n>d$. Provided that the sparsity exists, our screening rule will identify the inactive features and find the final active set in a finite number of iterations (Property 4). Thus, we still have $d_{k} < d $, and the per-iteration cost of our algorithm will be less than the original one. Thus, with fewer or the same iterations and lower per-iteration costs, the proposed algorithms are faster than the original ones for $n>d$.

Hence, by exploiting the sparsity in high-dimensional sparse learning, the computational and memory costs of both APGD and SPGD algorithms for Group OWL models are effectively reduced by our screening rule to a large extent.

\begin{figure*}[t]
\centering
\subfigure[PG]{\includegraphics[width=0.23\textwidth]{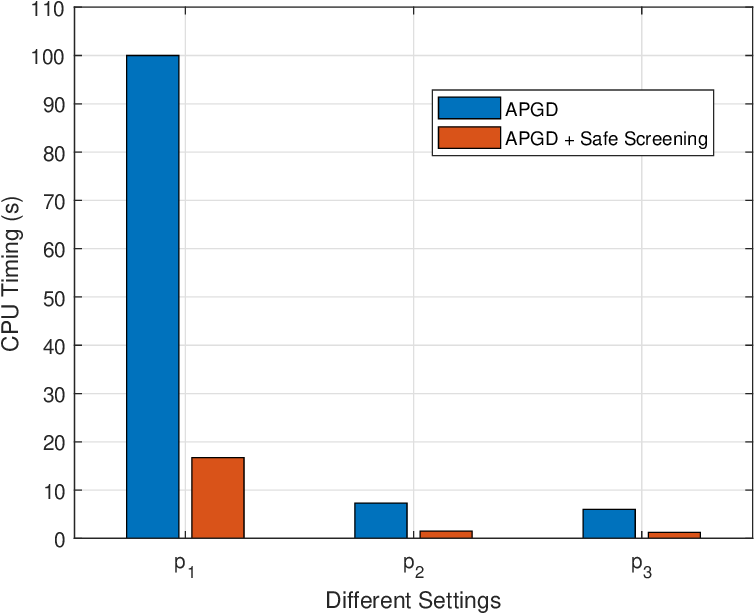}}
\hspace{2mm}
\subfigure[EG]{\includegraphics[width=0.23\textwidth]{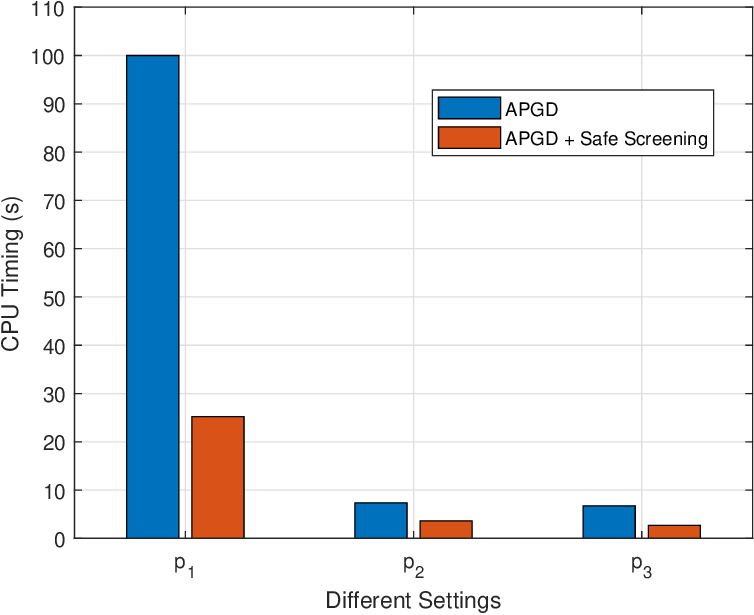}}
\hspace{2mm}
\subfigure[IL]{\includegraphics[width=0.23\textwidth]{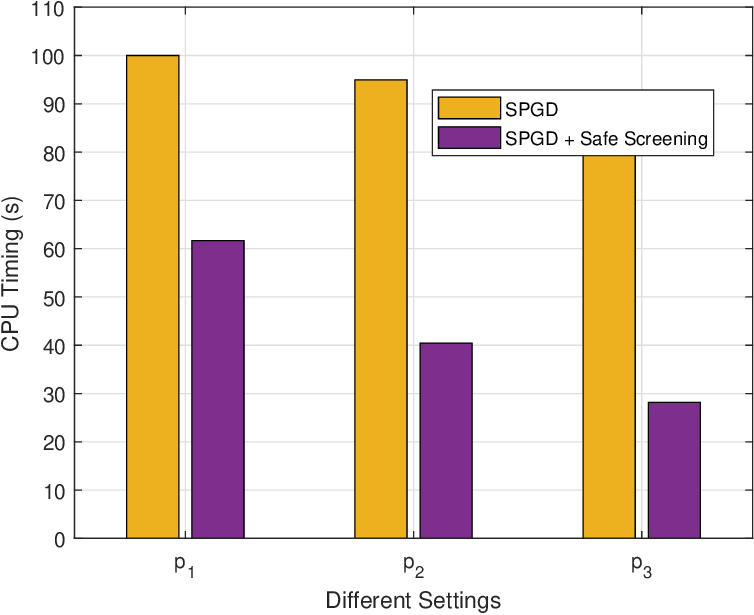}}
\hspace{2mm}
\subfigure[BM]{\includegraphics[width=0.23\textwidth]{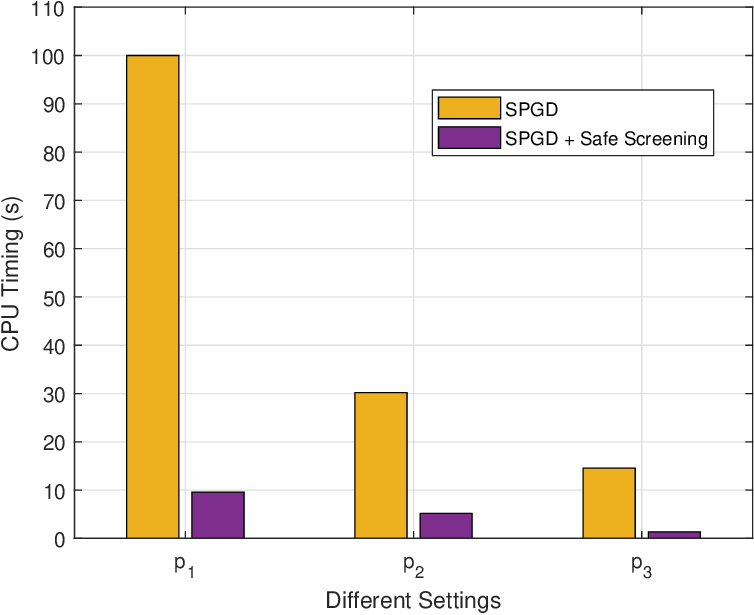}}
\caption{Running time of the algorithms without and with safe screening for Group OWL regression.}
\label{fig1}
\end{figure*}

\begin{figure*}[t]
\centering
\subfigure[PG]{\includegraphics[width=0.23\textwidth]{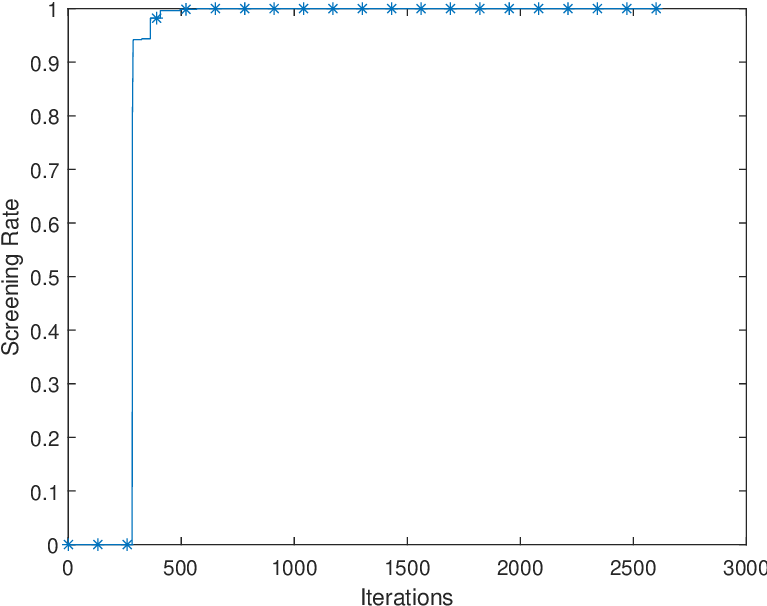}}
\hspace{2mm}
\subfigure[EG]{\includegraphics[width=0.23\textwidth]{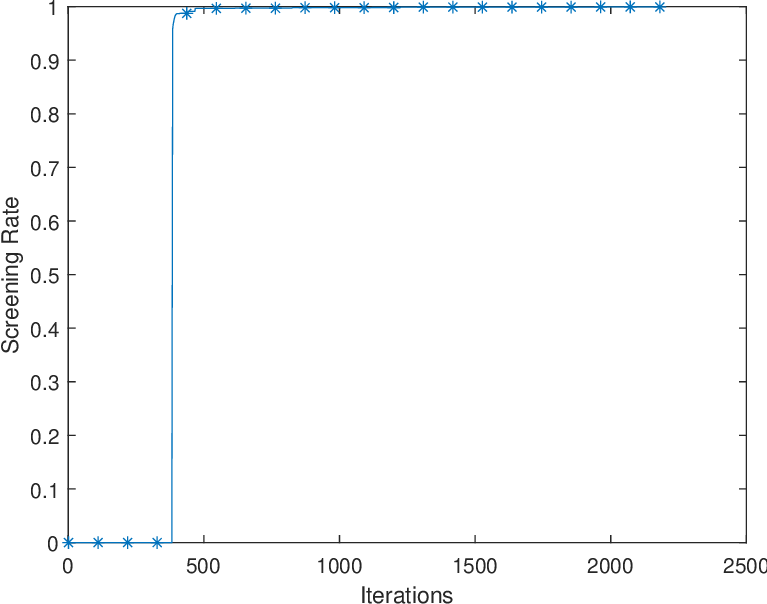}}
\hspace{2mm}
\subfigure[IL]{\includegraphics[width=0.23\textwidth]{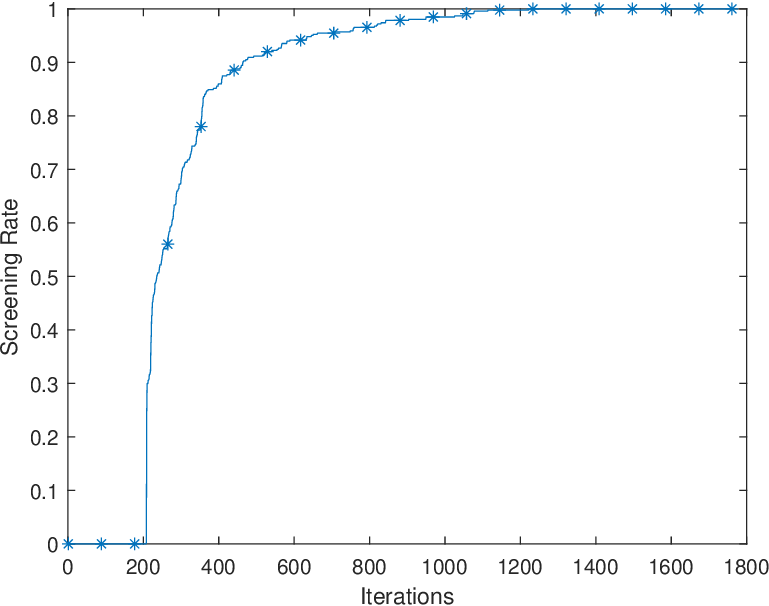}}
\hspace{2mm}
\subfigure[BM]{\includegraphics[width=0.23\textwidth]{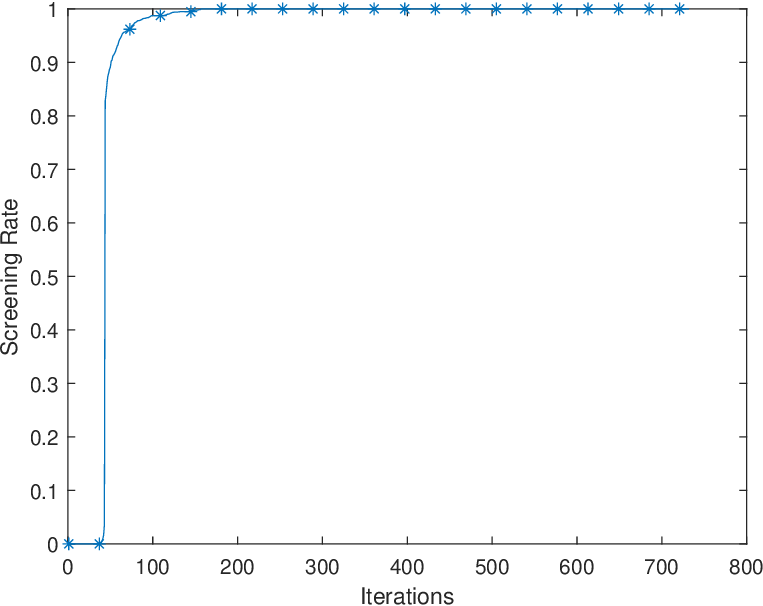}}
\caption{The screening rate of our screening rule in  the batch and stochastic settings for Group OWL regression.}
\label{fig2}
\end{figure*}

\subsection{Theoretical Analysis}

\begin{property}
The screening rule we proposed is guaranteed to be safe for the whole training process of Group OWL  models with any iterative optimization algorithm. \label{propertysafe}
\end{property}

Property \ref{propertysafe} shows our screening rule is safe to screen the features whose coefficients should be zero with the unknown dual optimum and the unknown order structure of the primal optimum for Group OWL models during the whole training process with any iterative optimization algorithm.

\begin{property} \label{propertyconverge}
Suppose iterative algorithm $\Psi$ to solve Group OWL models converges to the optimum, algorithm $\Psi$ with our screening rule to solve Group OWL models also converges to the optimum. 
\end{property}

Property \ref{propertyconverge} shows the convergence of standard iterative optimization algorithms with our screening rules can be guaranteed by the original algorithms. Thus, our screening rule can be directly applied to existing iterative optimization algorithms, \emph{e.g.}, APGD, SPGD, and their variants, with guaranteed convergence.

\begin{theorem} \label{propertyscreen}
Based on conditions (\ref{optimal}),  for any feature $ i $ belongs to the final active set $ \A^{*} $, we have that  $  \|x_i^\top \Theta^* \|_{2} \in [\min_{j\in  \A^*}\lambda_j, \max_{j\in \A^*}\lambda_j] $. Then, as algorithm $\Psi$ converges, there exists an iteration number $K_{0} \in \mathbb{N}$ s.t. $\forall k \geq K_{0}$, any feature $i \notin \A^{*}$ is screened by our screening rule. 
\end{theorem}

\begin{proof}
First, based on the strong concavity of the dual in Property 1, we know $\Theta^*$ is unique and $\Theta$ converges to $\Theta^*$ as $B$ converges to $B^*$, which means, as $\Psi$ converges, the intermediate duality gap converges towards zero during the training process. 

Thus, for any given tolerance $\epsilon$, there exists $K_{0}$ such that $\forall k \geq K_{0}$, we have 
\begin{eqnarray} \label{converge1}
    \| \Theta^{k} - \Theta^* \|_{F} \leq  \epsilon, \quad
 \sqrt{\frac{2G(B^k, \Theta^k)}{L}}\leq  \epsilon. 
\end{eqnarray}
For any $i \notin \A^{*}$, we have 
\begin{eqnarray} 
&&  \|x_i^\top \Theta^k \|_{2} +  \|x_i\|\sqrt{\frac{2G(B^k, \Theta^k)}{L}} \nonumber\\
& \leq & \|x_i\|\|\Theta^k-\Theta^*\|_{F} + \|x_i^\top \Theta^* \|_{2}   +   \|x_i\|\sqrt{\frac{2G(B^k, \Theta^k)}{L}} \nonumber\\
& \leq &  2 \|x_i \| \epsilon + \|x_i^\top \Theta^* \|_{2}.
\end{eqnarray}
The first inequality is obtained by the triangle inequality and the second inequality is obtained by (\ref{converge1}). 

Because  $i \notin \A^{*}$, we have $\lambda_{|\A^{*}|}-\|x_i^\top \Theta^* \|_{2} > 0 $. Thus, if we choose  
$
\epsilon < \frac{\lambda_{|\A^{*}|}-\|x_i^\top \Theta^* \|_{2} }{2\|x_i \|},
$
we have 
$
\|x_i^\top \Theta^k \|_{2} +  \|x_i\|\sqrt{{2G(B^k, \Theta^k)}/{L}}  < \lambda_{|\A^{*}|}, 
$
which is the screening rule we proposed. That is to say, feature $i$ is screened out by (\ref{screening}) at most at this iteration, which completes the proof.
\end{proof}

Theorem \ref{propertyscreen} shows the superb screening ability of our screening rules. Specifically, as the iterative algorithm converges, the intermediate duality gap becomes smaller. Correspondingly, the upper bound of the left term of our screening rule becomes tighter and the lower bound of the right term becomes larger. Hence, our screening rule is promising to screen more inactive variables. Finally, any inactive feature $i \notin \A^{*}$ is correctly detected and effectively screened by our screening rule in a finite number of iterations.

\section{Experiments}

\subsection{Experimental Setup}

\paragraph{Design of Experiments}
We conducted extensive experiments on real-world benchmark datasets for different Group OWL models not only to verify the effectiveness of our algorithm in reducing running time but also to show the effectiveness of the algorithms on screening inactive variables.

To validate the effectiveness of our algorithms in reducing running time, we evaluate the running time of our algorithms and other competitive algorithms to solve different Group OWL models under different setups. Group OWL models solved in our experiments include Group OWL regression and multinomial OWL regression.  Since the APGD algorithm performs well when $n \ll p $ and the SPGD  algorithm is proposed for large-scale learning where $n$ is large, for each model, we compare the running time in both batch and stochastic settings with different datasets respectively. Specifically, in the batch setting, we compare: 1) APGD; 2) APGD with our screening rules. In the stochastic setting, we compare: 1) SPGD, 2) SPGD with our screening rules.

To confirm the effectiveness of our algorithms on screening inactive variables, we evaluate the screening rate at each iteration of the algorithms with our screening rules for different Group OWL models in both batch and stochastic settings on the different datasets during the training process.

\begin{figure*}[t]
\centering
\subfigure[DF]{\includegraphics[width=0.23\textwidth]{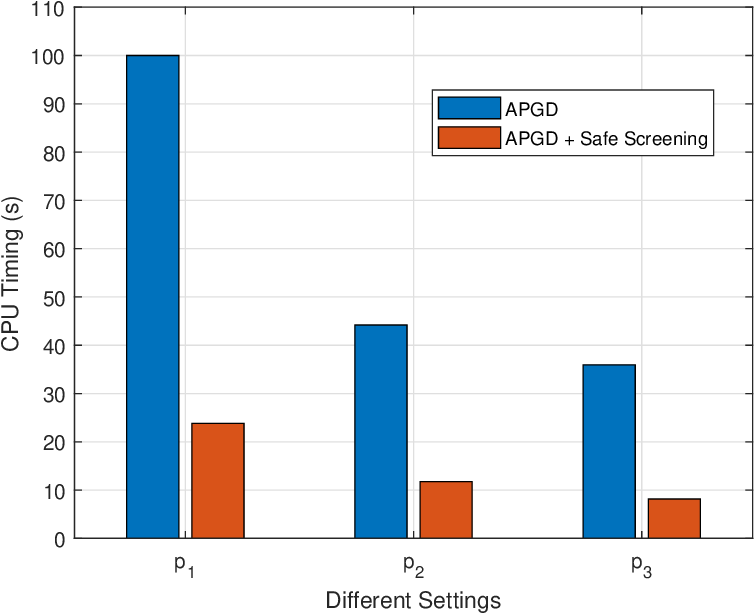}}
\hspace{2mm}
\subfigure[RE]{\includegraphics[width=0.23\textwidth]{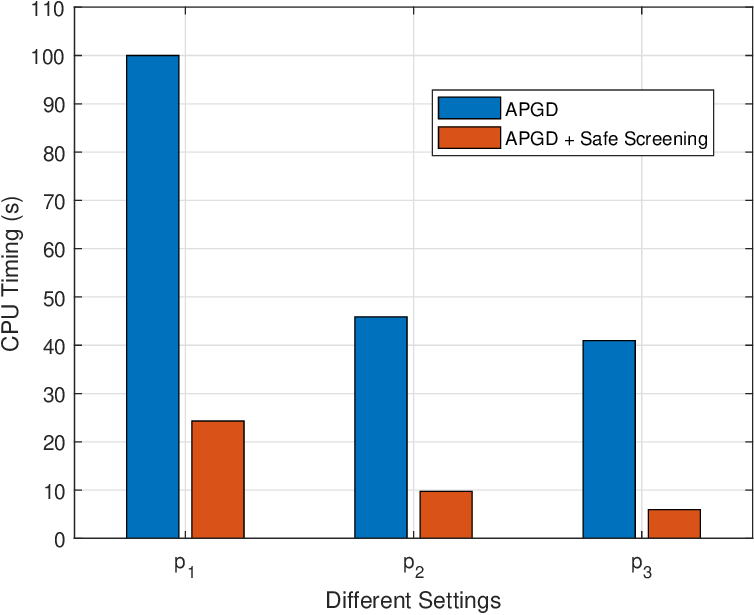}}
\hspace{2mm}
\subfigure[PR]{\includegraphics[width=0.23\textwidth]{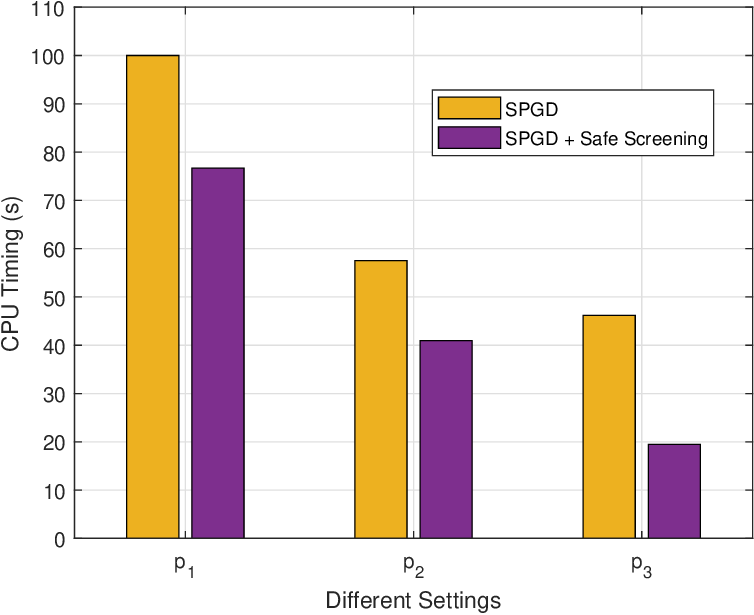}}
\hspace{2mm}
\subfigure[MN]{\includegraphics[width=0.23\textwidth]{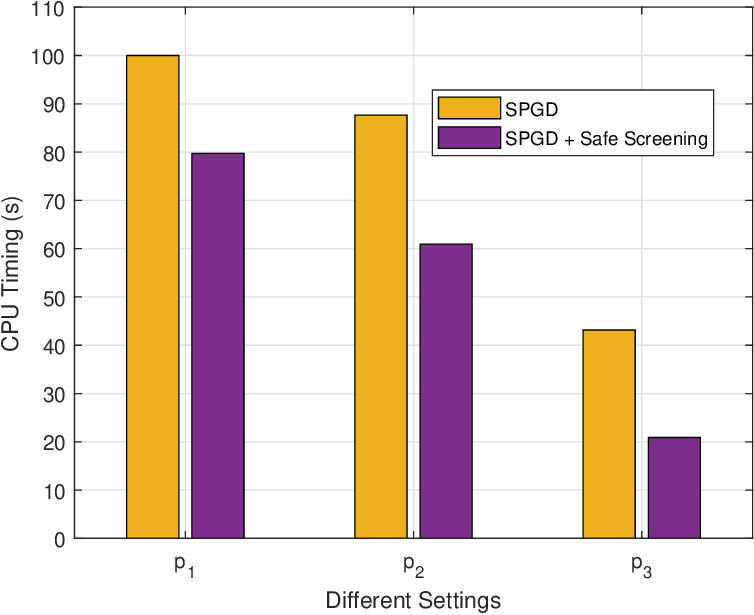}}
\caption{Running time of the algorithms without and with safe screening for multinomial OWL regression.}
\label{fig3}
\end{figure*}

\begin{figure*}[t]
\centering
\subfigure[DF]{\includegraphics[width=0.23\textwidth]{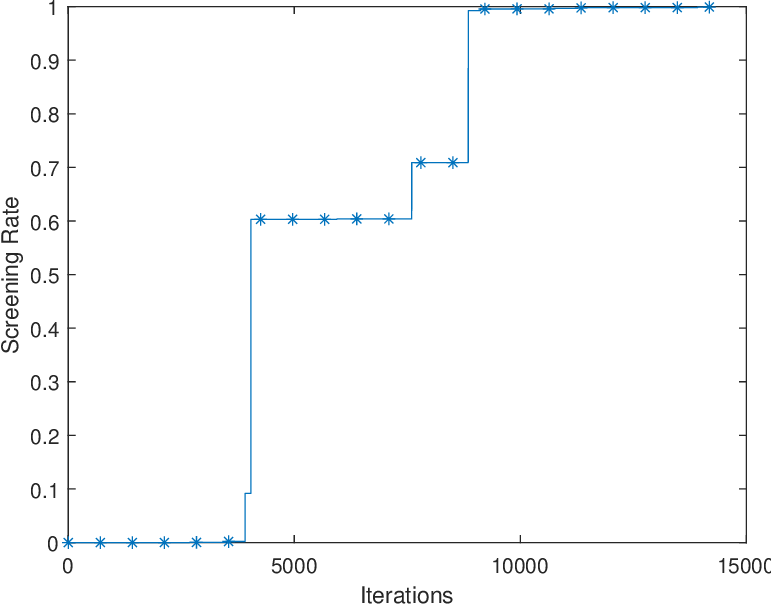}}
\hspace{2mm}
\subfigure[RE]{\includegraphics[width=0.23\textwidth]{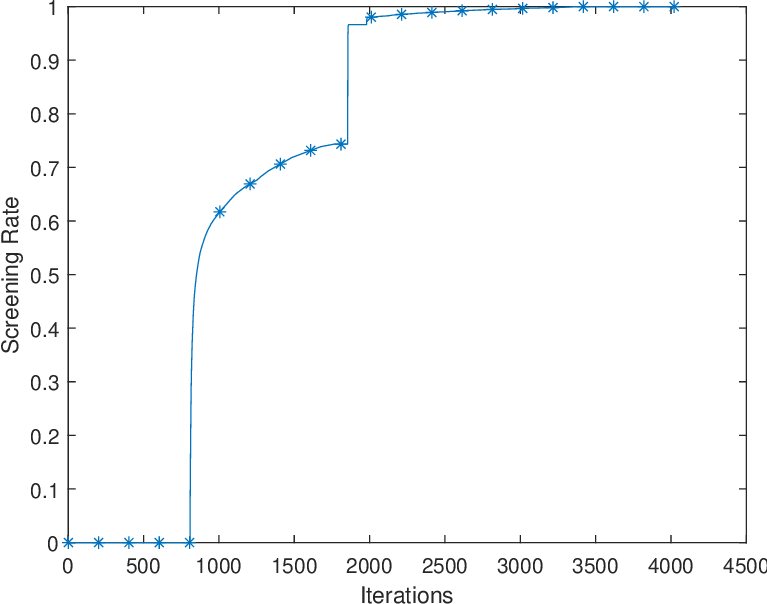}}
\hspace{2mm}
\subfigure[PR]{\includegraphics[width=0.23\textwidth]{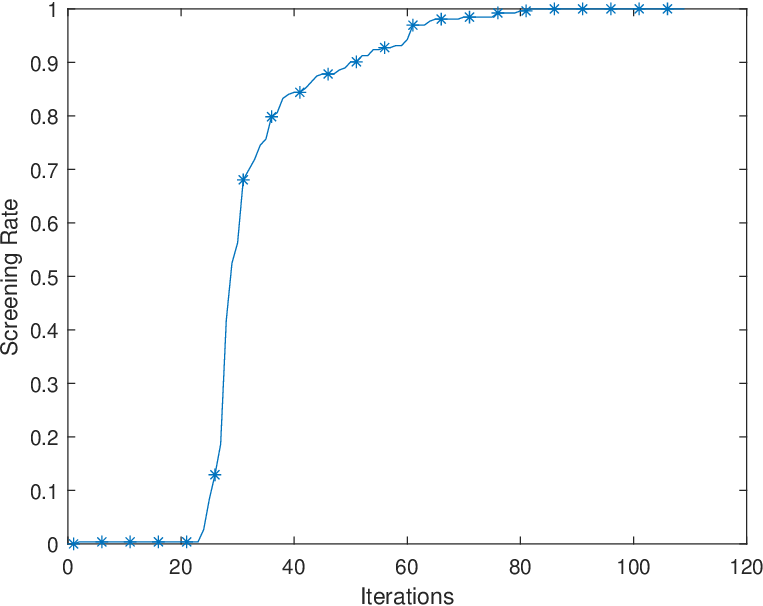}}
\hspace{2mm}
\subfigure[MN]{\includegraphics[width=0.23\textwidth]{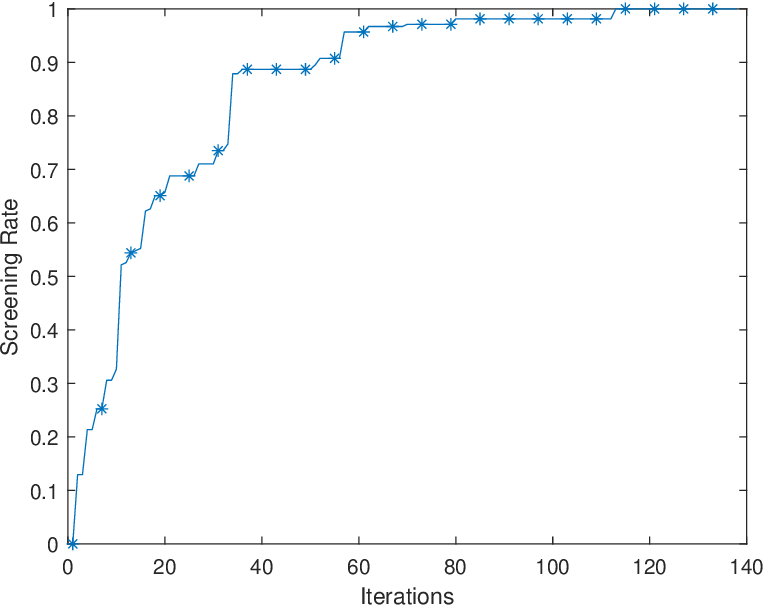}}
\caption{The screening rate of our screening rule in both batch and stochastic settings for multinomial OWL regression.}
\label{fig4}
\end{figure*}

\paragraph{Datasets}
Table \ref{table:datasets} summarizes the statistics of the benchmark datasets used in our experiments.  Protein and Mnist are from the LIBSVM repository \cite{chang2011libsvm}. DrivFace and IndoorLoc are from the UCI benchmark repository \cite{Dua:2019}. IndoorLoc has $2$ tasks, namely IndoorLoc Latitude and IndoorLoc Longitude. Reuters21578  is from \cite{cai2005document}. PlantGO, EukaryoteGO and Bookmarks are from \cite{xu2016multi, katakis2008multilabel} respectively.

\begin{table}[t]
\renewcommand{\arraystretch}{1.0}
\caption{The descriptions of benchmark datasets used.}
\begin{center}
\begin{tabular}{lcc}
\toprule
\textbf{Dataset} & \textbf{Sample Size} & \textbf{Attribute Size} \\
\midrule
    
     DrivFace (DF) & 606  & 6400   \\
    
     PlantGO (PG)    & 978 & 3091  \\
      
     EukaryoteGO (EG) & 4658  & 12689  \\
      
     Reuters21578 (RE)    & 5801 & 18933 \\

     Protein (PR)      & 17766 & 357 \\
    
     IndoorLoc (IL) & 21048 & 520  \\

     Mnist (MN) & 60000 & 780 \\

     Bookmarks (BM) & 87860 & 2150  \\

\bottomrule
\end{tabular}

%\end{small}
\end{center}
	\label{table:datasets}
\end{table}

\paragraph{Implementation Details}
We implement all the algorithms in MATLAB. We compare the average running CPU time on a 2.70 GHz machine for 5 trials. For the implementation, we follow the basic setting in \cite{bogdan2015slope,brzyski2019group} and set the tolerance error $\epsilon$ of the duality gap and the dual infeasibility as $10^{-6}$. For a fair comparison, all the experimental setups in Algorithm \ref{algAPGDScreen} and \ref{algSPGDScreen} follow the original APGD and SPGD algorithms with the same hyperparameters.  In the stochastic setting, for different datasets, the size of the mini-batch and the number of the inner loop are set as 30 or 40. Step size $\gamma$ is selected from $10^{-8}$ to $10^{-5}$. In the beginning, the duality gap of Group OWL models in APGD algorithms is large and thus the algorithm can barely benefit from our screening rule. We run our algorithms without screening at first and then apply our screening rule with a warm start. For comparison convenience, the CPU time of each algorithm is shown as the percentage of running time of the first setting for each dataset. 
OSCAR is a popular hyperparameter setting for the family of OWL norms \cite{oswal2016representational, zhong2012efficient, zhang2018learning}). We use the OSCAR setting for all our experiments:
    $\lambda_{i} = \alpha_{1} + \alpha_{2}(d-i),$
where $\alpha_{1} = p_{i} \|X^\top Y\|_{2,\infty}$ for Group OWL regression and multinomial OWL regression  and $\alpha_{2} = \alpha_{1}/d$. For a fair comparison, the factor $p_{i}$ is used to control the sparsity. In our experiments, we set $p_{i} = i * e^{-\tau}$, $i = 1,2, 3$. 

In Group OWL regression, we perform the batch algorithms on EukaryoteGO with $22$ tasks and PlantGo with $12$ tasks and the stochastic algorithm on IndoorLoc with $2$ tasks and Bookmarks with the first $20$ tasks. We set $\tau =2$ for Bookmarks, $\tau =4$ for EukaryoteGO  and $\tau =3$ for others. In multinomial OWL regression, we perform the batch algorithms on DrivFace with $3$ categories and Reuters21578  with $30$ categories and the stochastic algorithms on Protein with $3$ categories and Mnist with $10$ categories and set $\tau =2$ for Protein and Mnist and $\tau =3$ for DrivFace and Reuters21578. 

To evaluate the screening rate of our algorithms, the screening rate is computed as the proportion of the inactive variables screened by our screening test to the total inactive ones at the optimum. We choose setting $p_{1}$ for the evaluation of the screening rate.

\subsection{Experimental Results and Discussions}
\paragraph{Group OWL Regression}
Figures \ref{fig1}(a)-(d) plot the results of the average running time of the algorithms with and without our safe screening rule for Group OWL regression in the batch and stochastic settings in different situations respectively. When $n \ll d$, the results show, with our safe screening rule, the APGD algorithm achieves the computational gain to the original algorithm by a factor of 2x to 6x. For large-scale learning, the results show SPGD algorithm with our safe screening rule achieves the gain over the original one by a factor of 1.5x to 10x. The results confirm that the algorithms with our screening rule are always much faster than the original ones in both batch and stochastic settings for Group OWL regression. 

Figures \ref{fig2}(a)-(d) present the results of the screening rate of our algorithms in both batch and stochastic settings for Group OWL regression. The results show that our algorithm can successfully screen most of the inactive variables at the early stage, reach the final active set, and screen almost all the inactive variables in a finite number of iterations and thus is an effective method to screen inactive features for all the tasks of Group OWL regression.

\paragraph{Multinomial OWL Regression}
Figures \ref{fig3}(a)-(d) provide the results of the average running time of the algorithms with and without our safe screening rule for multinomial OWL regression in the batch and stochastic settings in different situations respectively. When $n \ll d$, the results show, with our safe screening rule, the APGD algorithm achieves the computational gain by a factor of 4x to 7x. For large-scale learning, the results show SPGD algorithm with our safe screening rule achieves the gain by a factor of 1.2x to 2.4x. The results confirm that the algorithms with our screening rule are always much faster than the original ones in both batch and stochastic settings for multinomial OWL regression. 

Figures \ref{fig4}(a)-(d) present the results of the screening rate of our algorithms in batch and stochastic settings for multinomial OWL regression. The results confirm the screening ability of our algorithm for multinomial OWL regression.

\section{Conclusion}
In this paper, we proposed safe screening rules for Group OWL models by overcoming the difficulties caused by the non-separable penalty, which can significantly accelerate the training process by avoiding the useless computations of the inactive variables. Specifically, our screening rule is doubly dynamic with a decreasing left term by tracking the intermediate duality gap and an increasing right term by exploring the order of the primal solution via an iterative strategy. More importantly, the proposed screening rules can be easily incorporated into existing iterative optimization algorithms in both batch and stochastic settings. Theoretically, we rigorously proved that our screening rule is not only safe for itself but also safe when applied to existing iterative optimization algorithms. Empirical results show our algorithms can lead to significant computational gains without any loss of accuracy by screening inactive variables.

\iffalse
\subsection{\LaTeX-Specific Advice}
Please don't use the \verb|{eqnarray}| equation environment. Use
\verb|{align}| or \verb|{IEEEeqnarray}| instead. The \verb|{eqnarray}|
environment leaves unsightly spaces around relation symbols.

Please note that the \verb|{subequations}| environment in {\LaTeX}
will increment the main equation counter even when there are no
equation numbers displayed. If you forget that, you might write an
article in which the equation numbers skip from (17) to (20), causing
the copy editors to wonder if you've discovered a new method of
counting.
\fi

\bibliographystyle{IEEEtran}
\bibliography{bibfile}

\end{document}